\documentclass[sigconf, screen, natbib=true,anonymous=false]{acmart}

\usepackage{hyperref}
\usepackage[noend]{algorithmic}
\usepackage{algorithm}
\usepackage{multirow}
\usepackage{amsmath}
\usepackage{bbm}
\usepackage{booktabs}
\usepackage[inline]{enumitem} 
\usepackage{subcaption}
\usepackage{bm} 
\usepackage{xcolor}

\usepackage{lipsum}

\usepackage{wrapfig}

\usepackage{arydshln}
\makeatletter
\def\adl@drawiv#1#2#3{
        \hskip.5\tabcolsep
        \xleaders#3{#2.5\@tempdimb #1{1}#2.5\@tempdimb}%
                #2\z@ plus1fil minus1fil\relax
        \hskip.5\tabcolsep}
\newcommand{\cdashlinelr}[1]{%
  \noalign{\vskip\aboverulesep
          \global\let\@dashdrawstore\adl@draw
          \global\let\ adl@draw\adl@drawiv}
  \cdashline{#1}
  \noalign{\global\let\adl@draw\@dashdrawstore
          \vskip\belowrulesep}}
\makeatother

\newcommand{\indep}{\perp \mkern-9.5mu \perp}

\usepackage{tikz}
\usetikzlibrary{automata, arrows, bayesnet, bending}

\usepackage{tikzscale}

\copyrightyear{2026}
\acmYear{2026}
\setcopyright{cc}
\setcctype{by}
\acmConference[SIGIR '26]{Proceedings of the 49th International ACM SIGIR Conference on Research and Development in Information Retrieval}{July 20--24, 2026}{Melbourne, VIC, Australia}
\acmBooktitle{Proceedings of the 49th International ACM SIGIR Conference on Research and Development in Information Retrieval (SIGIR '26), July 20--24, 2026, Melbourne, VIC, Australia}
\acmDOI{10.1145/3805712.3809837}
\acmISBN{979-8-4007-2599-9/2026/07}

\begin{document}
\title[Additive Control Variates Dominate Self-Normalisation]{Additive Control Variates Dominate Self-Normalisation~\\ in Off-Policy Evaluation}

\author{Olivier Jeunen}
\affiliation{
  \institution{aampe}
  \city{Antwerp}
  \country{Belgium}
}

\author{Shashank Gupta}
\authornote{Work done before joining Microsoft AI.}
\affiliation{
  \institution{Microsoft AI}
  \city{Bangalore}
  \country{India}
}

\begin{abstract}
Off-policy evaluation (OPE) is essential for assessing ranking and recommendation systems without costly online interventions.
Self-Normalised Inverse Propensity Scoring (SNIPS) is a standard tool for variance reduction in OPE, leveraging a multiplicative control variate.
Recent advances in off-policy learning suggest that additive control variates (baseline corrections) may offer superior performance, yet theoretical guarantees for evaluation are lacking.

This paper provides a definitive answer: we prove that $\beta^\star$-IPS, an estimator with an optimal additive baseline, asymptotically dominates SNIPS in Mean Squared Error.
By analytically decomposing the variance gap, we show that SNIPS is asymptotically equivalent to using a specific---but generally sub-optimal---additive baseline.
Our results theoretically justify shifting from self-normalisation to optimal baseline corrections for both ranking and recommendation.
\end{abstract}



\maketitle

\section{Introduction \& Motivation}
Search and recommender systems provide the most visible embodiment of Information Retrieval applications---serving a wide variety of use-cases, industries, and people.
A decision-centric view of such systems has recently emerged~\cite{Joachims2021}, leveraging ideas from the neighbouring fields of causal and counterfactual inference~\cite{CONSEQUENCES2022} to improve ranking and recommendation policies~\cite{Vasile2020,Saito2021}.
Off-Policy Evaluation and Learning (OPE/OPL) methods enable direct estimation and optimisation of online success metrics in an offline manner~\cite{Jeunen2021Thesis,Jeunen2023_nDCG}---as well as the treatment effect of a policy deployment~\cite{Jeunen2024_DeltaOPE}, which is typically a key performance criterion in real-world systems~\cite{kohavi2020trustworthy}.

Importance sampling and Inverse Propensity Scoring (IPS) are the cornerstones of practical OPE~\cite{vandenAkker2024}.
Standard IPS (the Horvitz--Thompson estimator~\cite{Horvitz1952}) enjoys unbiasedness at the cost of potentially high variance.
Self-normalised IPS (the H\'ajek estimator~\cite{Hajek1971}) leverages a multiplicative control variate to reduce this variance, at the cost of an asymptotically vanishing bias~\cite[\S 9.2]{Owen2013}.
Its effectiveness and parameter-free nature have rendered it a standard tool for OPE practitioners and researchers alike~\cite{Kong1992,Swaminathan2015, Gilotte2018,Joachims2018, Saito2021_OPE, gupta2025safe, gupta2024unbiased}, recently leading to proposals for off-policy ranking evaluation under the Item-Position Model~\cite{Li2018} as SNIPM~\cite{London2023}.

\citet{Joachims2018} show that, in OPL scenarios, the optimisation of the SNIPS objective can be equivalently reframed as optimising IPS with a specific additive control variate.
Building on this connection, \citet{Gupta2024} propose a generalised and unbiased estimator, $\beta$-IPS, and derive the optimal (variance-minimising) control variate $\beta^{\star}$.
Whilst their work empirically demonstrates the promising effectiveness of $\beta^{\star}$-IPS for OPE, definitive theoretical guarantees comparing its performance to self-normalisation remain absent. 

In this work, we bridge this theoretical gap between the common approach of self-normalisation and additive baseline corrections.
Our key contributions include:
\begin{enumerate}
    \item A formal proof that the $\beta^{\star}$-IPS estimator asymptotically dominates SNIPS in terms of Mean Squared Error (MSE).
    We show that SNIPS is asymptotically equivalent to an additive control variate estimator where the baseline is fixed to the true policy value $V(\pi)$, which is generally sub-optimal.
    \item A formal derivation that under linearisation (i.e. the Delta method), the asymptotic variance of SNIPS will always be lower-bounded by the variance of $\beta^{\star}$-IPS.
    \item An additive control variate estimator for rankings under the Item-Position Model: $\beta$-IPM, which is a natural extension of $\beta$-IPS to the ranking setting. We provide a formal proof that $\beta_{\indep}^{\star}$-IPM asymptotically dominates the SNIPM estimator in terms of MSE at every position in the ranking.
\end{enumerate}

\section{Off-Policy Evaluation, Background \& Notation}
We adopt an interventionist view to conceptualise general ranking applications, recommender systems, and Large Language Model (LLM) interfaces through the unified lens of \emph{policies}~\cite{Vasile2020,Joachims2021}.
A policy $\pi$ defines a conditional action distribution over actions $A$ (e.g. rankings, items, tokens), conditional on context $X$ (e.g. user profiles, session information, prompts): $\pi(a\mid x)\equiv\mathsf{P}(A=a\mid X=x,\Pi=\pi)$.
The \emph{value} of a policy is the expectation of the reward $R$ we obtain from deploying it:
\begin{equation}
    V(\pi) = \mathop{\mathbb{E}}\limits_{x \sim \mathsf{P}(X)}\left[\mathop{\mathbb{E}}\limits_{a \sim \pi(\cdot\mid x)}\left[ R \right ]  \right ].
\end{equation}
The reward $R$ here is general: it can comprise various explicit or implicit feedback signals, or multiple conflicting objectives~\cite{Sagtani2024,Jeunen2024_MultiObjective}.

We deal with \emph{off-policy} settings when we aim to estimate $V(\pi)$ from a dataset collected under a different policy $\pi_{0}$. The standard IPS estimator is unbiased under mild assumptions~\cite{Gilotte2018}:
\begin{equation}
    \hat{V}_{\rm IPS}(\pi,\mathcal{D}) = \frac{1}{|\mathcal{D}|}\sum_{(x,a,r)\in\mathcal{D}}\frac{\pi(a\mid x)}{\pi_0(a\mid x)}r.
\end{equation}
\citet{Swaminathan2015} suggest to leverage SNIPS for reduced variance, albeit at the cost of a finite-sample bias:
\begin{equation}
    \hat{V}_{\rm SNIPS}(\pi,\mathcal{D}) = \frac{\sum_{(x,a,r)\in\mathcal{D}}\frac{\pi(a\mid x)}{\pi_0(a\mid x)}r}{\sum_{(x,a,r)\in\mathcal{D}}\frac{\pi(a\mid x)}{\pi_0(a \mid x)}}.
\end{equation}
\citet{Joachims2018} demonstrate that direct optimisation of the $\hat{V}_{\rm SNIPS}$ objective is equivalent to optimising $\hat{V}_{\rm IPS}$ with a baseline correction, which \citet{Gupta2024} generalise to:
\begin{equation}
    \hat{V}_{\beta{\rm-IPS}}(\pi,\mathcal{D}) = \beta + \frac{1}{|\mathcal{D}|}\sum_{(x,a,r)\in\mathcal{D}}\frac{\pi(a\mid x)}{\pi_0(a\mid x)}(r-\beta).
\end{equation}
This estimator is unbiased, and the optimal (variance-minimising) correction $\beta^{\star}$ can be analytically computed from logged data $\mathcal{D}$.

The efficacy of an estimator is often measured through its Mean Squared Error (MSE), which decomposes as:
\begin{equation}
\mathrm{MSE}(\hat{V}) = \mathbb{E}\bigl[(\hat{V} - V(\pi))^2\bigr] = \mathrm{Bias}(\hat{V})^2 + \operatorname{Var}(\hat{V}).
\end{equation}

\section{Theoretical Contributions}
\subsection{$\beta^{\star}$-IPS dominates SNIPS in MSE}

\begin{theorem}[Asymptotic MSE comparison of $\beta^{\star}$-IPS and SNIPS]
\label{thm:mse-betaips-snips}
Let $\mathcal{D}=\{(x_i,a_i,r_i)\}_{i=1}^n$ be i.i.d.\ logged data generated under a logging
policy $\pi_0$. Let $w_i=\pi(a_i \mid x_i)/\pi_0(a_i \mid x_i)$ be the importance
weight for a target policy $\pi$. Assume:
\begin{enumerate}
    \item Rewards are bounded: $|r_i|\le R<\infty$.
    \item Importance weights are bounded: $0\le w_i\le W<\infty$.
    \item $\mathbb{E}[w_i]=1$ and $\mathbb{E}[w_i^2]<\infty$.
\end{enumerate}
Let $\beta^{\star}$ denote the baseline that minimises the mean squared error (MSE)
within the family of estimators with a global additive control variate
(cf.\ Eq.~(38) in~\cite{Gupta2024}). Then
\[
\mathrm{MSE}\!\left(\hat V_{\beta^{\star}\mathrm{-IPS}}\right)
\le
\mathrm{MSE}\!\left(\hat V_{\mathrm{SNIPS}}\right) + O(n^{-2}),
\]
and in particular $\beta^{\star}$-IPS is asymptotically no worse in MSE than SNIPS.
Moreover, if $\beta^{\star}\neq V(\pi)$ and $\mathrm{Var}(w_i)>0$, the inequality is
strict for all sufficiently large $n$.
\end{theorem}

\begin{proof}
We write $V=V(\pi)$ for brevity and define the empirical means
\[
\bar X \triangleq \frac{1}{n}\sum_{i=1}^n w_i r_i,
\qquad
\bar W \triangleq \frac{1}{n}\sum_{i=1}^n w_i .
\]
Then $\hat V_{\mathrm{SNIPS}}=\bar X/\bar W$ and
\[
\hat V_{\beta\mathrm{-IPS}}(\beta)
=
\beta + \frac{1}{n}\sum_{i=1}^n w_i(r_i-\beta)
=
\beta + \bar X - \beta\bar W .
\]
In particular, for $\beta=V$,
\begin{equation}
\hat V_{\beta\mathrm{-IPS}}(V)
=
V + \bar X - V\bar W .
\label{eq:betaipsV}
\end{equation}

\medskip
\noindent\textbf{Step 1: Exact decomposition.}
Let $L_n \triangleq \bar X - V\bar W$.
We obtain the exact identity
\begin{align}
\hat V_{\mathrm{SNIPS}} - \hat V_{\beta\mathrm{-IPS}}(V)
&=
\frac{\bar X}{\bar W} - \bigl(V+\bar X - V\bar W\bigr) \nonumber\\
&=
L_n\Bigl(\frac{1}{\bar W}-1\Bigr) \nonumber\\
&=
L_n\cdot \frac{1-\bar W}{\bar W}. \label{eq:Rn}
\end{align}
Define the remainder random variable
\begin{align}
R_n \triangleq L_n\cdot \frac{1-\bar W}{\bar W}.
\label{eq:rn}
\end{align}
Then
\begin{equation}\label{eq:snips_decomp}    
\hat V_{\mathrm{SNIPS}}
=
\hat V_{\beta\mathrm{-IPS}}(V) + R_n .
\end{equation}

\medskip
\noindent\textbf{Step 2: Bounding the remainder.}
Let $\mathcal{E}=\{\bar W\ge 1/2\}$. Since $0\le w_i\le W$ and
$\mathbb{E}[w_i]=1$, Hoeffding’s inequality~\cite[Thm.~2.8]{Boucheron2013} gives
\begin{align}
\mathbb{P}(\mathcal{E}^c)
=
\mathbb{P}(\bar W-1\le -1/2)
\le
\exp\Bigl(-\frac{n}{2W^2}\Bigr).
\label{eq:heof_ineq}
\end{align}
On $\mathcal{E}$ we have $\bar W\ge 1/2$ and hence
\[
\Bigl|\frac{1-\bar W}{\bar W}\Bigr|\le 2|1-\bar W|.
\]
Therefore, on $\mathcal{E}$,
\[
R_n^2 \le 4 L_n^2(1-\bar W)^2 .
\]
Splitting expectations on $\mathcal{E}$ yields
\begin{align}
\mathbb{E}[R_n^2]
&\le
4\,\mathbb{E}\!\left[L_n^2(1-\bar W)^2\right]
+
\mathbb{E}\!\left[R_n^2\mathbf{1}\{\mathcal{E}^c\}\right].
\label{eq:split}
\end{align}
To simplify notation, define the centred random variables:
\begin{align}
U_i \triangleq w_i(r_i - V), \qquad T_i \triangleq w_i - 1.
\label{eq:UTvar}
\end{align}
These variables correspond to the terms in our averages, such that $L_n = \frac{1}{n}\sum_{i=1}^n U_i$ and $1-\bar W = -\frac{1}{n}\sum_{i=1}^n T_i$.

Given the boundedness of rewards ($|r_i|\le R$) and weights ($0\le w_i\le W$), these new variables are also bounded:
\[
|U_i| \le RW(1+W), \qquad |T_i| \le W.
\]
Since $U_i$ and $T_i$ are zero-mean and bounded, we can bound the fourth moments of their averages. 
By Rosenthal's inequality (or standard moment expansions for i.i.d.\ sums, see e.g. \cite[Thm.~15.11]{Boucheron2013}), there exist finite constants $C_U, C_T$ such that:
\[
\mathbb{E}[L_n^4] \le \frac{C_U}{n^2}, \qquad \mathbb{E}[(1-\bar W)^4] \le \frac{C_T}{n^2}.
\]
We now bound the expectation of the product using the Cauchy--Schwarz inequality:
\[
\mathbb{E}\!\left[L_n^2(1-\bar W)^2\right]
\le
\sqrt{\mathbb{E}[L_n^4]} \sqrt{\mathbb{E}[(1-\bar W)^4]}
\le
\frac{\sqrt{C_U C_T}}{n^2}.
\]
Finally, observing that the case $\bar{W}=0$ occurs with exponentially vanishing probability, and combining the polynomial moment bound with the exponential decay of the failure probability $\mathbb{P}(\mathcal{E}^c)$, we conclude:
\[
\mathbb{E}[R_n^2] = O(n^{-2}).
\]

\medskip
\noindent\textbf{Step 3: MSE expansion.}
From $\hat V_{\mathrm{SNIPS}}=\hat V_{\beta\mathrm{-IPS}}(V)+R_n$, the MSE decomposes as
\begin{align*}
\mathrm{MSE}(\hat V_{\mathrm{SNIPS}})
&=
\mathbb{E}\!\left[\bigl(\hat V_{\beta\mathrm{-IPS}}(V)-V + R_n\bigr)^2\right] \\
&=
\mathrm{MSE}\!\left(\hat V_{\beta\mathrm{-IPS}}(V)\right)
+2\mathbb{E}\!\left[\bigl(\hat V_{\beta\mathrm{-IPS}}(V)-V\bigr)R_n\right]
+\mathbb{E}[R_n^2].
\end{align*}
From Eq.~\ref{eq:betaipsV}, $\hat V_{\beta\mathrm{-IPS}}(V)-V=L_n$, therefore the cross-term becomes\\
\[
    2\mathbb{E}\!\left[\bigl(\hat V_{\beta\mathrm{-IPS}}(V)-V\bigr)R_n\right]=2\mathbb{E}[L_n R_n]=2\mathbb{E}\left[L_n^2\frac{1-\bar W}{\bar W}\right].\
\]
We now define a bound on the cross term. Using $R_n \approx L_n(1-\bar{W})$ on the high-probability event $\mathcal{E}$ (as $\lim_{n\to\infty}\bar W =1$), the dominant component of the cross term is:
\[
\mathbb{E}[L_n^2(1-\bar{W})]
=
\mathbb{E}\!\left[ \left(\frac{1}{n}\sum_{i=1}^n U_i\right)^2 \left(-\frac{1}{n}\sum_{j=1}^n T_j\right) \right].
\]
Expanding the product of sums yields a triple summation:
\[
-\frac{1}{n^3} \sum_{i=1}^n \sum_{k=1}^n \sum_{j=1}^n \mathbb{E}[U_i U_k T_j].
\]
Since $U_i$ and $T_j$ are zero-mean variables, terms in this sum vanish unless all indices coincide (i.e., $i=k=j$).
For any distinct indices, the expectation factorises and zeroes out (e.g., if $j \neq i,k$, then $\mathbb{E}[U_i U_k]\mathbb{E}[T_j]=0$).
Only $n$ diagonal terms remain:
\[
\mathbb{E}[L_n^2(1-\bar{W})]
=
-\frac{1}{n^3} \sum_{i=1}^n \mathbb{E}[U_i^2 T_i]
=
O(n^{-2}).
\]
This confirms that both the cross-term and $\mathbb{E}[R_n^2]$ decay at a rate of $O(n^{-2})$, and hence the total approximation error is also $O(n^{-2})$:
\[
\mathrm{MSE}(\hat V_{\mathrm{SNIPS}})
=
\mathrm{MSE}\!\left(\hat V_{\beta\mathrm{-IPS}}(V)\right)
+ O(n^{-2}).
\]

\medskip
\noindent\textbf{Step 4: Comparison with the optimal baseline.}
By definition, $\beta^{\star}$-IPS minimises the variance---and hence MSE---within the
class of offline estimators with a global additive control variate. Therefore,
\[
\mathrm{MSE}\left(\hat{V}_{\beta-\mathrm{IPS}}(\beta^{\star})\right)
\le
\mathrm{MSE}\!\left(\hat V_{\beta\mathrm{-IPS}}(V)\right).
\]
Combining with the previous step gives
\[
\mathrm{MSE}\!\left(\hat V_{\beta^{\star}\mathrm{-IPS}}\right)
\le
\mathrm{MSE}\!\left(\hat V_{\mathrm{SNIPS}}\right) + O(n^{-2}).
\]
Finally, if $\beta^{\star}\neq V$ and $\mathrm{Var}(w_i)>0$, the gap between the two
additive estimators is of order $1/n$ (from the difference in variance), whereas the remainder term is of order
$1/n^2$. This implies strict dominance of $\beta^{\star}$-IPS over SNIPS for all
sufficiently large $n$.
\end{proof}

\subsection{$\beta^{\star}$-IPS reduces SNIPS' asymptotic variance}
\begin{proposition}[Analytical Variance Gap]
\label{prop:variance-gap}
Let $\sigma_w^2 = \mathrm{Var}(w_i)$, $\sigma_{wr}^2 = \mathrm{Var}(w_i r_i)$, and $\sigma_{w,wr} = \mathrm{Cov}(w_i, w_i r_i)$.
The variance of the $\beta$-IPS estimator is
\[
\mathrm{Var}\!\left(\hat V_{\beta\mathrm{-IPS}}(\beta)\right)
= \frac{1}{n}\Bigl( \sigma_{wr}^2 - 2\beta \sigma_{w,wr} + \beta^2 \sigma_w^2 \Bigr),
\]
which is minimised by $\beta^{\star} = \sigma_{w,wr} / \sigma_w^2$.
The asymptotic variance of SNIPS (derived via the Delta method~\cite[Eq.~9.8]{Owen2013}) is
\[
\mathrm{AVar}\!\left(\hat V_{\mathrm{SNIPS}}\right)
= \frac{1}{n}\Bigl( \sigma_{wr}^2 - 2V(\pi) \sigma_{w,wr} + V(\pi)^2 \sigma_w^2 \Bigr).
\]
The variance gap $\Delta \triangleq \mathrm{AVar}(\hat V_{\mathrm{SNIPS}}) - \mathrm{Var}(\hat V_{\beta^{\star}\mathrm{-IPS}})$ is given exactly by:
\begin{equation}
\Delta = \frac{1}{n}\, \frac{\bigl(V(\pi)\sigma_w^2 - \sigma_{w,wr}\bigr)^2}{\sigma_w^2} \ge 0.
\end{equation}
Consequently, $\hat V_{\beta^{\star}\mathrm{-IPS}}$ strictly dominates SNIPS in variance unless $V(\pi) = \beta^{\star}$.
\end{proposition}

\begin{proof}
Substituting $\beta^{\star} = \sigma_{w,wr}/\sigma_w^2$ into the variance expression for $\beta$-IPS yields the optimal variance:
\begin{align}
\mathrm{Var}\!\left(\hat V_{\beta^{\star}\mathrm{-IPS}}\right)
&= \frac{1}{n}\left( \sigma_{wr}^2 - 2\frac{\sigma_{w,wr}}{\sigma_w^2}\sigma_{w,wr} + \left(\frac{\sigma_{w,wr}}{\sigma_w^2}\right)^2 \sigma_w^2 \right) \nonumber \\
&= \frac{1}{n}\left( \sigma_{wr}^2 - \frac{\sigma_{w,wr}^2}{\sigma_w^2} \right). \label{eq:opt-var}
\end{align}
Subtracting~\eqref{eq:opt-var} from the SNIPS asymptotic variance yields:
\begin{align*}
\Delta
&= \frac{1}{n} \left( V(\pi)^2\sigma_w^2 - 2V(\pi)\sigma_{w,wr} + \frac{\sigma_{w,wr}^2}{\sigma_w^2} \right) \\
&= \frac{1}{n\sigma_w^2} \left( V(\pi)^2\sigma_w^4 - 2V(\pi)\sigma_w^2 \sigma_{w,wr} + \sigma_{w,wr}^2 \right) \\
&= \frac{1}{n\sigma_w^2} \bigl( V(\pi)\sigma_w^2 - \sigma_{w,wr} \bigr)^2 .
\end{align*}
Since $\sigma_w^2 > 0$ and the numerator is a perfect square, $\Delta \ge 0$. The gap is zero if and only if $V(\pi)\sigma_w^2 = \sigma_{w,wr}$, which is equivalent to $V(\pi) = \sigma_{w,wr}/\sigma_w^2 = \beta^{\star}$.
\end{proof}

\subsection{$\beta_{\indep}^\star$-IPM dominates SNIPM at every position}
Real-world applications often present users with a ranked list of results, instead of selecting a single action.
Considering rankings as atomic actions leads to a combinatorial explosion of the action space, prohibiting the direct use of IPS-based estimators~\cite{gupta2023safe}.
Assumptions about how users interact with rankings then give rise to structured estimators with significantly lower variance.

Variance reduction techniques have recently been applied to the pseudoinverse estimator for slates~\cite{Swaminathan2017}, relying on factored policies and focusing on slate-level rewards~\cite{Vlassis2021}.
In contrast, our work relates to general ranking policies and position-level rewards.

The Item-Position Model (IPM)~\cite{Li2018} is also common in the literature, resting on the assumption that an observed reward at position $j$ is independent of other items in the ranked list.
\citet{London2023} provide a seminal study into the use of self-normalised estimators for ranking policies, introducing the SNIPM estimator.
Our analysis naturally extends to include (SN)IPM.

Consider a ranking of size $k$, and let $\bm{a} = (a_1,\ldots,a_k)$ and $\bm{r} = (r_1,\ldots,r_k)$ represent vectorised actions and rewards per position.
The SNIPM estimator is defined as the sum of position-wise SNIPS estimators~\cite[Eq. 1]{London2023}:
\begin{equation}
    \hat{V}_{\rm SNIPM}(\pi,\mathcal{D}) = \sum_{j=1}^{k} \hat{V}^{(j)}_{\rm SNIPS}(\pi,\mathcal{D}) =
    \sum_{j=1}^{k} \frac{\sum\limits_{(x,\bm{a},\bm{r})\in\mathcal{D}}\frac{\pi(a_{j}\mid x,j)}{\pi_0(a_{j}\mid x,j)}r_{j}}{\sum\limits_{(x,\bm{a},\bm{r})\in\mathcal{D}}\frac{\pi(a_{j}\mid x,j)}{\pi_0(a_{j} \mid x, j)}}.
\end{equation}
It is natural to extend the $\beta$-IPS estimator analogously to SNIPM as a local $\bm{\beta}$-IPM estimator where $\bm{\beta} = (\beta_1, \dots, \beta_k)$:
\begin{gather}
    \hat{V}_{\bm{\beta}\rm -IPM}(\pi,\mathcal{D}) =  \sum_{j=1}^{k} \hat{V}^{(j)}_{{\beta_j}\rm-IPS}(\pi,\mathcal{D}) = \nonumber\\
    \sum_{j=1}^{k} \left( \beta_j + \frac{1}{|\mathcal{D}|}\sum_{(x,\bm{a},\bm{r})\in\mathcal{D}}\frac{\pi(a_j\mid x,j)}{\pi_0(a_j\mid x,j)}(r_j-\beta_j)\right).
\end{gather}
Optimising the control variate $\beta_j$ at every position independently is not guaranteed to lead to a global optimum without restrictive ranking policy assumptions---such as the factored policies used by \citet{Vlassis2021}. To derive the globally optimal solution $\bm{\beta}^{\star}$, we would need to account for the dependencies across different ranking positions, which would introduce additional estimation complexity. We leave this as a promising area for future work, and instead propose a robust approximation $\beta_{\indep}^{\star}$ that ignores the dependencies between different positions.

\noindent Following \citet{Gupta2024}, the components $\beta_{\indep, j}^{\star}$ are derived independently for every ranking position as:
\begin{equation}
    \beta_{\indep , j}^{\star} 
    = \frac{\text{Cov}(w_j, w_j r_j)}{\text{Var}(w_j)} 
    = \frac{\mathbb{E}[w_j^2 r_j] - \mathbb{E}[w_j r_j]}{\mathbb{E}[w_j^2] - 1}.
\end{equation}
Even this natural and straightforward extension of the $\beta$-IPS idea comes with strong guarantees regarding asymptotic optimality.

\begin{theorem}[Asymptotic MSE comparison of $\beta_{\indep}^{\star}$-IPM and SNIPM]
Let $\beta_{\indep,j}^\star$ be the baseline that minimises the variance of the $\beta$-IPS estimator at position $j$.
Then, for every position $j \in \{1, \dots, k\}$, the component estimator $\hat{V}^{(j)}_{\beta_{\indep,j}^\star\text{-IPS}}$ asymptotically dominates the corresponding component of SNIPM:
\begin{equation}
    \mathrm{MSE}\!\left(\hat{V}^{(j)}_{\beta_{\indep,j}^\star\text{-IPS}}\right)
    \le
    \mathrm{MSE}\!\left(\hat{V}^{(j)}_{\mathrm{SNIPS}}\right) + O(n^{-2}).
\end{equation}
Consequently, $\bm{\beta}_{\indep}^\star$-IPM achieves position-wise asymptotic dominance over SNIPM at every ranking position.
\end{theorem}

\begin{proof}
The proof follows by applying Theorem~\ref{thm:mse-betaips-snips} to each position $j$ individually.
The SNIPM estimator at rank $j$ is structurally identical to SNIPS applied to the marginal distribution of items and rewards at that rank.
From our main result, we established that $\hat{V}_{\mathrm{SNIPS}}^{(j)}$ is asymptotically equivalent to an additive estimator with a fixed baseline equal to the true expected reward at that position, $\beta_j = V_j(\pi)$.
In contrast, $\hat{V}_{\beta_{\indep,j}^\star\text{-IPS}}^{(j)}$ utilises the variance-minimising baseline for that specific position.
Therefore, unless the optimal baseline at rank $j$ coincides exactly with $V_j(\pi)$, the $\bm\beta_{\indep}^{\star}$-IPS component strictly reduces variance compared to the SNIPM component for every position $j\in\{1,\ldots,k\}$.
\end{proof}

\subsection{A practical note on estimation bias for $\beta^{\star}$}
Whilst $\hat{V}_{\beta\text{-IPS}}$ is strictly unbiased for any fixed (data-independent) baseline $\beta$, plugging in the empirical variance-minimiser $\hat{\beta}^{\star}$ estimated from the same dataset $\mathcal{D}$ introduces a correlation between the baseline and the sample estimates.
This results in a finite-sample bias of order $O(n^{-1})$---the same order as the bias inherent to SNIPS.
Consequently, even without correction, the empirical $\hat{\beta}^{\star}$-IPS estimator maintains the same asymptotic MSE convergence rate as SNIPS whilst benefiting from the superior variance constant derived in Proposition~\ref{prop:variance-gap}.
For applications requiring strict finite-sample unbiasedness, one can employ $k$-fold cross-fitting~\cite{Chernozhukov2018}: partitioning $\mathcal{D}$ into $k$ folds, estimating $\hat{\beta}^{\star}$ on $1$ fold, and computing the value estimate on the remaining $k-1$ folds---averaging results.
This procedure restores exact unbiasedness at the cost of a negligible reduction in effective sample size for the variance estimation.
Note that similar optimisation strategies for control variates have been explored in the broader context of off-policy evaluation with slates~\cite{Vlassis2021}.

\subsection{Discussion \& Practical Takeaway}
Our main results are asymptotic, but have direct implications for finite-sample behaviour.
In particular, Proposition~\ref{prop:variance-gap} characterises the exact asymptotic variance gap between additive control variates and SNIPS.
By contrast, Theorem~\ref{thm:mse-betaips-snips} shows that the discrepancy between SNIPS and $\beta$-IPS with $\beta = V(\pi)$ is governed by a remainder term that decays as $O(n^{-2})$ under standard boundedness assumptions.
As a result, the leading $O(n^{-1})$ variance term dominates the MSE for sufficiently large $n$, favouring $\beta^{\star}$-IPS.

What ``sufficiently large'' entails is problem-dependent.
When the logging and target policies are close, $\mathrm{Var}(w)$ is small and the practical difference between SNIPS and $\beta$-IPS can be negligible.
Conversely, in settings with substantial policy mismatch---precisely the realistic regimes where variance reduction is most critical---the variance gap in Proposition~\ref{prop:variance-gap} can be large, implying that $\beta^\star$-IPS attains lower MSE at moderate sample sizes.
The empirical results reported by \citet{Gupta2024} and \citet{Jeunen2024_DeltaOPE} corroborate this theoretical insight.

Our observations suggest that whilst SNIPS remains a stable and convenient default in very small-sample regimes, additive control variates with an estimated optimal baseline offer a strictly better bias–variance trade-off once a modest amount of logged data is available.
These results therefore provide theoretical support for replacing self-normalisation with optimal baseline corrections in practical off-policy evaluation pipelines.

\section{Conclusions \& Outlook}
In this work, we have provided a rigorous theoretical comparison between two dominant approaches for variance reduction in off-policy evaluation: self-normalisation (SNIPS) and additive baseline corrections ($\beta$-IPS).
We proved that the estimator utilising the optimal additive control variate, $\beta^{\star}$-IPS, asymptotically dominates SNIPS in terms of Mean Squared Error.
Furthermore, by deriving the exact analytical variance gap between the two, we revealed that SNIPS is asymptotically equivalent to an additive estimator with a fixed baseline of $\beta=V(\pi)$.
Since the optimal baseline $\beta^{\star}$ coincides with the policy value $V(\pi)$ only under specific distributional conditions, SNIPS is inherently sub-optimal in the general case.
We demonstrated that this dominance extends to the structured ranking setting, by proving that under the standard IPM, $\beta_{\indep}^{\star}$-IPM achieves position-wise asymptotic dominance over SNIPM.

These findings suggest that the community's reliance on self-normalisation as a parameter-free default may be misplaced.
Since $\beta^{\star}$ can be efficiently estimated from logged data with minimal bias, $\beta^{\star}$-IPS and its ranking counterpart offer a superior bias-variance trade-off without requiring complex hyper-parameter tuning.
Finally, while our analysis focused on scalar policy evaluation, these results naturally extend to pairwise comparisons~\cite{Jeunen2024_DeltaOPE}, where variance reduction is even more critical for detecting statistically significant improvements between policies.


\balance

\bibliographystyle{ACM-Reference-Format}
\bibliography{bibliography}

\end{document}